\def\year{2018}\relax
\documentclass[letterpaper]{article} 
\usepackage{aaai18}  
\usepackage{times}  
\usepackage{helvet}  
\usepackage{courier}  
\usepackage{url}  
\usepackage{graphicx}  
\usepackage[utf8]{inputenc} 
\usepackage[T1]{fontenc}    
\usepackage[american]{babel}
\usepackage[hidelinks,draft]{hyperref}       
\usepackage{url}            
\usepackage{booktabs}       
\usepackage{amsfonts}       
\usepackage{nicefrac}       
\usepackage{microtype}      
\usepackage[dvipsnames]{xcolor}
\usepackage{tikz}
\usepackage{amsmath}
\usepackage{amssymb}
\usepackage{amsxtra}
\usepackage{amsthm}
\usepackage{array}
\usepackage{algorithm2e}
\usepackage{caption}
\usepackage{multirow}
\usepackage{siunitx}
\usepackage{subfig}
\usepackage{graphicx}
\usepackage{threeparttable} 
\usepackage{footnote}
\usepackage{diagbox}
\usepackage{adjustbox}
\usepackage[autolanguage]{numprint}

\usepackage{calc}

\usepackage{caption}
\usepackage{titlesec}
\titlespacing{\paragraph}{%
  1em}{
  0.\baselineskip}{
  .5em} 

\graphicspath{{figure/}}

\newcommand{\eat}[1]{} 

\newcommand{\sz}[1]{\lvert#1\rvert}   

\newcommand{\eqdef}{\stackrel{\mathrm{def}}{=}} 

\newcommand{\floor}[1]{\left\lfloor#1\right\rfloor} 

\newcommand{\absn}[1]{\lvert#1\rvert} 

\newcommand{\td}[2]{\if*#1\else^{#1}\fi\if*#2\else_{#2}\fi} 

\newcommand{\sset}[1]{\left\{\,#1\,\right\}} 

\newcommand\join\Join 

\DeclareSymbolFont{txsymbolsC}{U}{txsyc}{m}{n}
\SetSymbolFont{txsymbolsC}{bold}{U}{txsyc}{bx}{n}
\DeclareFontSubstitution{U}{txsyc}{m}{n}
\DeclareMathSymbol{\ljoin}{\mathrel}{txsymbolsC}{88}
\DeclareMathSymbol{\rjoin}{\mathrel}{txsymbolsC}{89}

\newsavebox\setminusbox
\newlength\setminuslen

\newcommand\xdiag{\operatorname{diag}}
\newcommand\diag[1]{\xdiag\left(#1\right)}    

\newcommand\norm[1]{\left\lVert #1 \right\rVert}

\newcolumntype{C}{>{$\displaystyle}c<{$}} 
\newcolumntype{L}{>{$\displaystyle}l<{$}} 
\newcolumntype{R}{>{$\displaystyle}r<{$}} 
\newcolumntype{H}{>{\setbox0=\hbox\bgroup}c<{\egroup}@{}} 

\makeatletter
\renewcommand*\env@matrix[1][*\c@MaxMatrixCols c]{%
  \hskip -\arraycolsep
  \let\@ifnextchar\new@ifnextchar
  \array{#1}}
\makeatother

\newcommand{\B}[3]{B\if*#1\else_{#1}\fi(#2,#3)} 
\newcommand{\I}[3]{I\if*#1\else_{#1}\fi(#2,#3)} 

\makeatletter
\def\imod#1{\allowbreak\mkern10mu({\operator@font mod}\,\,#1)}
\makeatother

\newlength\hspaceoflen

\newcommand\vect[1]{{\boldsymbol{#1}}}
\newcommand\va{\vect{a}}
\newcommand\vb{\vect{b}}
\newcommand\vc{\vect{c}}
\newcommand\vd{\vect{d}}
\newcommand\ve{\vect{e}}
\newcommand\vf{\vect{f}}
\newcommand\vg{\vect{g}}
\newcommand\vh{\vect{h}}
\newcommand\vi{\vect{i}}
\newcommand\vj{\vect{j}}
\newcommand\vk{\vect{k}}
\newcommand\vl{\vect{l}}
\newcommand\vm{\vect{m}}
\newcommand\vn{\vect{n}}
\newcommand\vo{\vect{o}}
\newcommand\vp{\vect{p}}
\newcommand\vq{\vect{q}}
\newcommand\vr{\vect{r}}
\newcommand\vs{\vect{s}}
\newcommand\vt{\vect{t}}
\newcommand\vu{\vect{u}}
\newcommand\vv{\vect{v}}
\newcommand\vw{\vect{w}}
\newcommand\vx{\vect{x}}
\newcommand\vy{\vect{y}}
\newcommand\vz{\vect{z}}
\newcommand\vzero{\vect{0}} 
\newcommand\vone{\vect{1}}

\newcommand\mA{\vect{A}}
\newcommand\mB{\vect{B}}
\newcommand\mC{\vect{C}} 
\newcommand\mD{\vect{D}}
\newcommand\mE{\vect{E}}
\newcommand\mF{\vect{F}}
\newcommand\mG{\vect{G}}
\newcommand\mH{\vect{H}}
\newcommand\mI{\vect{I}}
\newcommand\mJ{\vect{J}}
\newcommand\mK{\vect{K}}
\newcommand\mL{\vect{L}}
\newcommand\mM{\vect{M}}
\newcommand\mN{\vect{N}} 
\newcommand\mO{\vect{O}}
\newcommand\mP{\vect{P}}
\newcommand\mQ{\vect{Q}} 
\newcommand\mR{\vect{R}} 
\newcommand\mS{\vect{S}}
\newcommand\mT{\vect{T}}
\newcommand\mU{\vect{U}}
\newcommand\mV{\vect{V}}
\newcommand\mW{\vect{W}}
\newcommand\mX{\vect{X}}
\newcommand\mY{\vect{Y}}
\newcommand\mZ{\vect{Z}}

\newcommand\bN{\mathbb{N}} 

\newcommand\bR{\mathbb{R}} 

\DeclareMathAlphabet{\mathcal}{OMS}{cmsy}{m}{n}

\newcommand\cE{\mathcal{E}}

\newcommand\cK{\mathcal{K}}

\newcommand\cM{\mathcal{M}}

\newcommand\cR{\mathcal{R}}

\newcommand\cT{\mathcal{T}}

\DeclareMathAlphabet\mathbfcal{OMS}{cmsy}{b}{n}

\newcommand\tB{\mathbfcal{B}}

\newcommand\tP{\mathbfcal{P}}

\newcommand\tS{\mathbfcal{S}}

\newcommand\tX{\mathbfcal{X}}

\accentedsymbol\Abar{{\bar A}}
\accentedsymbol\Bbar{{\bar B}}
\accentedsymbol\Cbar{{\bar C}}
\accentedsymbol\Dbar{{\bar D}}
\accentedsymbol\Ebar{{\bar E}}
\accentedsymbol\Fbar{{\bar F}}
\accentedsymbol\Gbar{{\bar G}}
\accentedsymbol\Hbar{{\bar H}}
\accentedsymbol\Ibar{{\bar I}}
\accentedsymbol\Jbar{{\bar J}}
\accentedsymbol\Kbar{{\bar K}}
\accentedsymbol\Lbar{{\bar L}}
\accentedsymbol\Mbar{{\bar M}}
\accentedsymbol\Nbar{{\bar N}}
\accentedsymbol\Obar{{\bar O}}
\accentedsymbol\Pbar{{\bar P}}
\accentedsymbol\Qbar{{\bar Q}}
\accentedsymbol\Rbar{{\bar R}}
\accentedsymbol\Sbar{{\bar S}}
\accentedsymbol\Tbar{{\bar T}}
\accentedsymbol\Ubar{{\bar U}}
\accentedsymbol\Vbar{{\bar V}}
\accentedsymbol\Wbar{{\bar W}}
\accentedsymbol\Xbar{{\bar X}}
\accentedsymbol\Ybar{{\bar Y}}
\accentedsymbol\Zbar{{\bar Z}}

\accentedsymbol\abar{{\bar a}}
\accentedsymbol\bbar{{\bar b}}
\accentedsymbol\cbar{{\bar c}}
\accentedsymbol\dbar{{\bar d}}
\accentedsymbol\ebar{{\bar e}}
\accentedsymbol\fbar{{\bar f}}
\accentedsymbol\gbar{{\bar g}}
\makeatletter
\@ifundefined{hbar}{}{
        \let\hbar\@undefined
}
\makeatother
\accentedsymbol\hbar{{\bar h}}
\accentedsymbol\ibar{{\bar i}}
\accentedsymbol\jbar{{\bar j}}
\accentedsymbol\kbar{{\bar k}}
\accentedsymbol\lbar{{\bar l}}
\accentedsymbol\mbar{{\bar m}}
\accentedsymbol\nbar{{\bar n}}
\makeatletter
\@ifundefined{obar}{}{
        \let\obar\@undefined      
}
\makeatother
\accentedsymbol{\obar}{{\bar o}}        
\accentedsymbol\pbar{{\bar p}}
\accentedsymbol\qbar{{\bar q}}
\accentedsymbol\rbar{{\bar r}}
\accentedsymbol\sbar{{\bar s}}
\accentedsymbol\tbar{{\bar t}}
\accentedsymbol\ubar{{\bar u}}
\accentedsymbol\vbar{{\bar v}}
\accentedsymbol\wbar{{\bar w}}
\accentedsymbol\xbar{{\bar x}}
\accentedsymbol\ybar{{\bar y}}
\accentedsymbol\zbar{{\bar z}}

\accentedsymbol\mAhat{{\hat\mA}}
\accentedsymbol\mBhat{{\hat\mB}}
\accentedsymbol\mChat{{\hat\mC}}
\accentedsymbol\mDhat{{\hat\mD}}
\accentedsymbol\mEhat{{\hat\mE}}
\accentedsymbol\mFhat{{\hat\mF}}
\accentedsymbol\mGhat{{\hat\mG}}
\accentedsymbol\mHhat{{\hat\mH}}
\accentedsymbol\mIhat{{\hat\mI}}
\accentedsymbol\mJhat{{\hat\mJ}}
\accentedsymbol\mKhat{{\hat\mK}}
\accentedsymbol\mLhat{{\hat\mL}}
\accentedsymbol\mMhat{{\hat\mM}}
\accentedsymbol\mNhat{{\hat\mN}}
\accentedsymbol\mOhat{{\hat\mO}}
\accentedsymbol\mPhat{{\hat\mP}}
\accentedsymbol\mQhat{{\hat\mQ}}
\accentedsymbol\mRhat{{\hat\mR}}
\accentedsymbol\mShat{{\hat\mS}}
\accentedsymbol\mThat{{\hat\mT}}
\accentedsymbol\mUhat{{\hat\mU}}
\accentedsymbol\mVhat{{\hat\mV}}
\accentedsymbol\mWhat{{\hat\mW}}
\accentedsymbol\mXhat{{\hat\mX}}
\accentedsymbol\mYhat{{\hat\mY}}
\accentedsymbol\mZhat{{\hat\mZ}}

\accentedsymbol\vahat{{\hat\va}}
\accentedsymbol\vbhat{{\hat\vb}}
\accentedsymbol\vchat{{\hat\vc}}
\accentedsymbol\vdhat{{\hat\vd}}
\accentedsymbol\vehat{{\hat\ve}}
\accentedsymbol\vfhat{{\hat\vf}}
\accentedsymbol\vghat{{\hat\vg}}
\accentedsymbol\vhhat{{\hat\vh}}
\accentedsymbol\vihat{{\hat\vi}}
\accentedsymbol\vjhat{{\hat\vj}}
\accentedsymbol\vkhat{{\hat\vk}}
\accentedsymbol\vlhat{{\hat\vl}}
\accentedsymbol\vmhat{{\hat\vm}}
\accentedsymbol\vnhat{{\hat\vn}}
\accentedsymbol\vohat{{\hat\vo}}
\accentedsymbol\vphat{{\hat\vp}}
\accentedsymbol\vqhat{{\hat\vq}}
\accentedsymbol\vrhat{{\hat\vr}}
\accentedsymbol\vshat{{\hat\vs}}
\accentedsymbol\vthat{{\hat\vt}}
\accentedsymbol\vuhat{{\hat\vu}}
\accentedsymbol\vvhat{{\hat\vv}}
\accentedsymbol\vwhat{{\hat\vw}}
\accentedsymbol\vxhat{{\hat\vx}}
\accentedsymbol\vyhat{{\hat\vy}}
\accentedsymbol\vzhat{{\hat\vz}}

\newtheorem{theorem}{Theorem}

\newcommand\rrank{\operatorname{rrank}}
\newcommand\round{\operatorname{round}}
\newcommand\Real{\operatorname{Re}}
\theoremstyle{definition}
\newtheorem{definition}{Definition}
\newtheorem{corollary}{Corollary}

\newcommand{\citet}[1]{\citeauthor{#1}~\shortcite{#1}}

\newif\ifshowcomments
\showcommentstrue

\ifshowcomments
\newcommand{\mynote}[2]{\fbox{\bfseries\sffamily\scriptsize{#1}}
 {\small$\blacktriangleright$\textsf{\emph{#2}}$\blacktriangleleft$}}
\else
\newcommand{\mynote}[2]{}
\fi

\usepackage{comment}
\newif\ifshowcomments
\showcommentsfalse

\ifshowcomments%
\newcommand{\incom}[1]{#1}%
\newcommand{\notcom}[1]{}%
\else
\excludecomment{comments}%
\newcommand{\incom}[1]{}
\newcommand{\notcom}[1]{#1}%
\fi

\newif\ifshowdetail
\showdetailfalse
\ifshowdetail%
\newcommand{\indetail}[1]{#1}%
\newcommand{\notdetail}[1]{}%
\else
\excludecomment{detail}%
\newcommand{\indetail}[1]{}%
\newcommand{\notdetail}[1]{#1}%
\fi

\newcommand\smod{\ \operatorname{mod} \ }

\begin{document}

\title{On Multi-Relational Link Prediction with Bilinear Models}
\author{Yanjie Wang$^\dagger$, Rainer Gemulla$^\dagger$, Hui Li$^\ddagger$\\
$^\dagger$University of Mannheim\\
$^\ddagger$The University of Hong Kong\\
}
\maketitle
\begin{abstract}
  We study bilinear embedding models for the task of multi-relational link
  prediction and knowledge graph completion. Bilinear models belong to the most
  basic models for this task, they are comparably efficient to train and use,
  and they can provide good prediction performance. The main goal of this paper
  is to explore the expressiveness of and the connections between various
  bilinear models proposed in the literature. In particular, a substantial
  number of models can be represented as bilinear models with certain additional
  constraints enforced on the embeddings. We explore whether or not these
  constraints lead to \emph{universal models}, which can in principle represent
  every set of relations, and whether or not there are \emph{subsumption
    relationships} between various models. We report results of an independent
  experimental study that evaluates recent bilinear models in a common
  experimental setup. Finally, we provide evidence that relation-level ensembles
  of multiple bilinear models can achieve state-of-the art prediction
  performance.
\end{abstract}

\section{Introduction}
Multi-relational link prediction is the task of predicting missing links in an
edge-labeled graph. We focus and use the terminology of \emph{knowledge base
  completion} throughout. Large-scale knowledge bases (KB) such as
DBPedia~\cite{LehmannIJJKMHMK15} or YAGO~\cite{MahdisoltaniBS15} contain millions of entities
and facts, but they are nevertheless far from being complete~\cite{Nickel0TG16}. Given
a set of entities (vertices) and relations (edge labels) that hold between these
entities, the goal of multi-relational link prediction~\cite{BordesUGWY13} is to
determine whether or not some entity $e_1$ links to some entity $e_2$ via a
relation $R$, i.e., whether the fact $R(e_1,e_2)$ is true.

Embedding models have recently received considerable attention for knowledge
base completion tasks~\cite{BordesUGWY13,NickelRP16,TrouillonWRGB16}. Such models embed both entities and relations
in a low-dimensional latent space such that the structure of the knowledge base
is (largely) maintained. The embeddings are subsequently used to predict missing
facts or to detect erroneous facts.

The perhaps most basic class of embedding models is given by bilinear
models. Such models predict a ``score'' for each fact $R(e_1,e_2)$ by computing
a weighted sum---where the weights depend on $R$---of the pairwise interactions of
the entity embeddings of $e_1$ and $e_2$. The scores are used to rank (pairs of)
entities according to their predicted truthfulness. Bilinear models are comparably
efficient to train and use and they can provide good prediction
performance~\cite{TrouillonN17}.

A large number of bilinear models has been proposed in the literature, including
RESCAL~\cite{NickelTK11}, TransE~\cite{BordesUGWY13},
DISTMULT~\cite{YangYHGD14a}, HolE~\cite{NickelRP16}, and
ComplEx~\cite{TrouillonWRGB16}. There is, however, little work on the expressiveness of and
the connections between various bilinear models. In this paper, we argue that
all of the aforementioned models can be seen as bilinear models subject to
certain constraints. We study whether and under which conditions each model is
\emph{universal} in that it can represent every possible set of relation
instances (or, more precisely, entity rankings). We also explore the size of the
embeddings needed for universality and derive upper bounds for the embedding
size needed to obtain embeddings consistent with a given dataset. We establish a
number of subsumption relationships between various models by giving explicit
constructions on how to transform instances of one model to instances of another
model (sometimes with a different embedding size). A summary of our results is
given in Tab.~\ref{tab:summary}.

We report on an independent experimental study that compared various bilinear
models on standard datasets in a common experimental setup. We found that the
relative performance among the models is highly relation-dependent. We thus
propose a simple relation-level ensemble of multiple bilinear models,
which---according to our experiments---significantly and consistently improved
prediction performance over individual models. In fact, we found that the
ensemble performed competitively to the state-of-the-art embedding approaches,
whether or not they are bilinear.


\section{Multi-Relational Link Prediction}

Let $\cE$ and $\cR$ be a set of entities and relation names. A knowledge base
$\cK\subseteq\cE\times\cR\times\cE$ is a collection of triples $(i, k, j)$ where
$i$, $j$, and $k$ refer to subject, object and relation, resp. We denote by
$K=\sz{\cR}\ge1$ and $N=\sz{\cE}\ge2$ the number of entities and relations, resp. We
represent knowledge base $\cK$ via a binary tensor
$\tX\in\{0,1\}^{N\times N\times K}$, where $x_{ijk}=1$ if and only if
$(i,k,j)\in\cK$. By convention, vectors $\va_i$ refer to rows of matrix $\mA$ (as
a column vector) and scalars $a_{ij}$ to individual entries. Given
dimensionalities $r$ and $r'$, we denote by $\ve_{i,r}$ the $i$-th standard
basis vector, by $\vzero_r$ the zero vector, and by $\vzero_{r\times r'}$ the zero
matrix of the respective shape. Finally, let $\diag{\cdot}$ refer to a
block-diagonal matrix built from the arguments (a vector or a list of matrices).

\subsection{Preliminaries}

A \emph{score-based ranking model} is a model $m$ that associates a \emph{score}
$s_k^m(i,j) \in \mathbb{R}$ with each subject-relation-object triple. Denote by
$\mS_k^m\in\bR^{N\times N}$ the corresponding \emph{scoring matrix} for relation
$k$, i.e., $[\mS_k^m]_{ij} = s_k^m(i,j)$. Denote by
$\tS^m\in\bR^{N\times N\times K}$ the \emph{scoring tensor} of $m$, i.e., the tensor with
frontal slices $\tS^m_{(k)}=\mS_k^m$.

We are ultimately interested in rankings, not in scores. In particular,
score-based models are used to rank (pairs of) entities by their predicted
truthfulness, given a query of form $R(i,?)$, $R(?,j)$, or $R(?,?)$. Generally,
a result with a higher score is considered more likely to be correct. We say
that an $N\times N$ matrix is a \emph{ranking matrix} if all its entries are in
$\sset{1,2,\ldots,N^2}$ and whenever there is any entry with value $s>1$, there
is at least one other entry with value $s-1$. Denote by $\pi(\mS)$ the unique
ranking matrix associated with scoring matrix $\mS$, where
$\pi_{ij}(\mS)\eqdef [\pi(\mS)]_{ij}$ is the \emph{dense rank} of $s_{ij}$ in
the multiset of the entries of $\mS$. For every pair of tuples
$(i,j)\in N\times N$ and $(i',j')\in N\times N$, we have
\[
  s_{ij} \le s_{i'j'} \iff \pi_{ij}(\mS) \ge \pi_{i'j'}(\mS).
\]
For example, 
\[
\mS=\begin{pmatrix}
0.2 & 2.4 & 1\\ 
-1 & 4 & 2 \\ 
-3 & 0.2 & 0  
\end{pmatrix}
\quad\implies\quad
\pi(\mS)=\begin{pmatrix}
5 & 2 & 4\\ 
7 & 1 &  3\\ 
8 & 5 & 6  
\end{pmatrix}
\]

In a slight abuse of notation, we overload $\pi$ to also apply to tensors, sets of
matrices, and sets of tensors. In particular, the \emph{ranking tensor}
$\pi(\tS)$ for a score tensor $\tS$ is the $N\times N\times K$ tensor produced from
$\tS$ by replacing every frontal slice $\tS_{(k)}$ with
$\pi(\tS_{(k)})$. Moreover, for any set $X$, set
$\pi(X)=\sset{\pi(x) : x\in X}$. Observe that $\pi(\bR^{N\times N})$ corresponds to the set
of all possible ranking matrices, $\pi(\bR^{N\times N\times K})$ to all possible ranking
tensors, and that $\pi(-\mP)=\mP$ for any ranking matrix (or ranking tensor)
$\mP$.

\subsection{Bilinear Models}

\emph{Bilinear models} are models whose scoring function $s_k(i,j)$ has form
$\va_i^T\mR_k\va_j$, where $\va_i,\va_j\in\bR^r$ and $\mR_k\in\bR^{r\times r}$ are model
parameters and are referred to as the \emph{embeddings} of entities $i$ and $j$
as well as relation $k$, resp. We refer to $r\in\bN$ as the \emph{size} of the
model.

In this paper, we consider bilinear models as well as models that can be
represented as bilinear models with an at most linear increase in model
size. Although some of the model considered here may not ``look'' bilinear at
first glance, we show that they are closely related to bilinear models. We
denote throughout the set of all models of type $t$ (and of size $r$) and by
$M^t$ ($M^t_r$).

\paragraph{RESCAL~\protect\cite{NickelTK11}.} An unconstrained bilinear model. Each
model $m\in M^{\text{RESCAL}}_{r}$ is parameterized by an entity matrix
$\mA\in\mathbb{R}^{N\times r}$ and $K$ relation matrices
$\mR_1,\ldots,\mR_K\in\mathbb{R}^{r\times r}$. We have
\[
s_k^{m}(i,j) = \va_i^T\mR_k\va_j.
\]
RESCAL can be seen as an extension of the low-rank matrix factorization methods
prominent in recommender systems to more then one relation.

\paragraph{DISTMULT~\protect\cite{YangYHGD14a}.} Each model
$m\in M^{\text{DISTMULT}}_{r}$ is parameterized by an entity matrix
$\mA\in\mathbb{R}^{N\times r}$ and a relation matrix $\mR\in\bR^{K\times r}$. We have
\[
  s_k^{m}(i,j) = \va_i^T\diag{\vr_k}\va_j.
\]
DISTMULT can be seen as a variant of RESCAL that puts a diagonality constraint
on the relation matrices. Due to this constraint, it can only model symmetric
relations. The model is equivalent to the INDSCAL tensor
decomposition~\cite{Carroll1970}.

\paragraph{HolE~\protect\cite{NickelRP16}.} Each model $m\in M^{\text{HolE}}_r$ is
parameterized by an entity matrix $\mA\in\bR^{N\times r}$ and a relation matrix
$\mR\in\bR^{K\times r}$. We have
\[
s_k^{m}(i,j) = \vr_k^T (\va_i \star \va_j),
\]
where $\star$ refers to the \emph{circular correlation} between $\va_i$ and $\va_j$,
i.e., $(\va_i \star \va_j)_k=\sum_{t=1}^{r} a_{it} a_{j((k+t-2 \mod
  r)   + 1)}$. 
   The idea of using circular
convolution relates to associative memory \cite{NickelRP16}. \citet{HayashiS17}
provide an alternative viewpoint in terms of ComplEx, discussed next.

\paragraph{ComplEx~\protect\cite{TrouillonWRGB16}.} Each model
$m\in M^{\text{ComplEx}}_{r}$ is parameterized by an entity matrix
$\mA\in\mathbb{C}^{N\times r}$ and a relation matrix $\mR\in\mathbb{C}^{N\times r}$. We have
\[
s_k^{m}(i,j) = \Real(\va_i^T\diag{\vr_k}\va_j),
\]
where $\Real(\cdot)$ extracts the real part of a complex number. ComplEx is
superficially related to DISTMULT but uses complex-valued parameter
matrices. Note that $\va_i^T\diag{\vr_k}\va_j$ is not guaranteed to be real.

\paragraph{TransE~\protect\cite{BordesUGWY13}.} Each model $m\in M^{\text{TransE}}_r$ is parameterized by an
entity matrix $\mA\in\bR^{N\times r}$ and an relation matrix
$\mR\in\bR^{K\times r}$. We have\footnote{This definition differs from the original
  definition of TransE in that we negate all scores in order to rank larger
  scores higher.}
\[
  s_k^{m}(i,j) = -\norm{\va_i+\vr_k - \va_j}^2_2.
\]
In contrast to the models presented above, TransE is a translation-based model,
not a factorization-based model. The use of translations---i.e., differences
between entity embeddings---is inspired by Word2Vec's word analogy
results~\cite{abs-1301-3781}. Note that TransE can also be used with $L_1$ norm
instead of $L_2$; we focus on the $L_2$ variant given above throughout.


\section{Subsumption and Expressiveness}

For a given class $M^t_r$ of models, denote by
$\cM^t_r=\sset{\tS^m : m\in M^t_r}$ the set of scoring tensors that the model
class can represent. Let $\cM^t=\cup_{r\in \bN^+} \cM^t_r$. Note that
$\pi(\cM^t_r)$ and $\pi(\cM^t)$ denote the set of ranking tensors that can be
represented by $M^t_r$ and $M^t$, respectively.

\begin{table*}[h]
  \caption{Summary of our main results. Each row corresponds to a model of size
    $r$. All conditions are sufficient conditions. ? means that no bound other
    than the universal bound is known.}
  \centering
  \label{tab:summary}
  \begin{adjustbox}{max width=\textwidth}
    \begin{tabular}{l@{\hspace{.7em}}c@{\hspace{.7em}}c@{\hspace{.7em}}c@{\hspace{.7em}}c@{\hspace{0.5em}}c@{\hspace{0.5em}}c@{\hspace{0.5em}}c@{\hspace{0.5em}}c}
      \hline
      \multirow{2}{*}{Model} & \multirow{2}{*}{\# Parameters} & Universal & Consistent with $\tB$              & \multicolumn{5}{c}{Subsumption of model of size $r'$ when $r\ge$} \\ \cline{5-9}
                             &                                & when $r\ge$ & when $r\ge$                               & RESCAL & HolE & ComplEx & DISTMULT & TransE                     \\ \hline
      RESCAL                 & $Nr+Kr^2$                      & $N$       & $\min\{N, 2\sum_{k} \rrank(\mB_{k})\}$     & $r'$   & $r'$ & $2r'+1$ & $r'$     & $2r'+1$                    \\
      HolE                   & $Nr+Kr$                        & $2KN+1$   & $2\min\{KN, 2\sum_{k} \rrank(\mB_{k})\}+1$ & ?      & $r'$ & $2r'+1$ & $2r'+1$  & ?                          \\
      ComplEx                & $2Nr+2Kr$                      & $KN$      & $\min\{KN, 2\sum_{k} \rrank(\mB_{k})\}$    & ?      & $r'$  & $r'$    & $r'$      & ?                          \\
      DISTMULT               & $Nr+Kr$                        & No        & No                                      & No     & No   & No      & $r'$     & No                         \\
      TransE                 & $Nr+Kr$                        & No        & No                                      & No     & No   & No      & No       & $r'$                       \\ \hline
    \end{tabular}
  \end{adjustbox}
\end{table*}

\subsection{Subsumption}

We first explore subsumption relationships between different model classes as
well as the the size of the entity representations needed for a subsumption to
hold. We assume throughout that the number $N\ge2$ of entities and the number
$K\ge1$ of relations are arbitrary but fixed.

We say that class $M^{t_2}$ \emph{subsumes} class $M^{t_1}$ whenever
$\pi(\cM^{t_1}) \subseteq \pi(\cM^{t_2})$. In other words, $M^{t_2}$ is at least
as expressive in terms of rankings as $M^{t_1}$. If
$\pi(\cM^{t_1}) \subset \pi(\cM^{t_2})$, we say that $M^{t_2}$ \emph{strictly
  subsumes} $M^{t_1}$, indicating that $M^{t_2}$ is strictly more expressive
than $M^{t_1}$. Note that it is good when $M^{t_2}$ is more expressive than
$M^{t_1}$ because $M^{t_2}$ can in principle express more rankings. It can also
be problematic, however, because efficient training and the avoidance of
overfitting become more challenging.

We first show subsumption by specifying an explicit model transformation, then
strictness via a counterexample.
\begin{theorem}\label{thm:R_from_T}
  For all $r\in\bN^+$, $M^{\text{RESCAL}}_{2r+1}$ subsumes
  $M^{\text{TransE}}_r$.
\end{theorem}
\begin{proof}
  Fix some $r\in\bN^+$. Pick any TransE model $m_T\in M^{\text{TransE}}_r$,
  denote by $\mA\in\bR^{N\times r}$ and $\mR\in\bR^{K\times r}$ the
  corresponding parameter matrices, and by $\tS^{m_T}$ the scoring tensor. We
  show that $\pi(\tS^{m_T})\in \pi(\cM^\text{RESCAL}_{2r+1})$. We do this by
  explicitly constructing a corresponding RESCAL model
  $m_R\in M^\text{RESCAL}_{2r+1}$ by specifying its parameters
  $\mA'\in\bR^{N\times (2r+1)}$ and $\mR_k'\in\bR^{(2r+1)\times
    (2r+1)}$. Setting
  \begin{equation} \label{eq:R_from_T}
    \begin{split}
      \va_i' &=
      \begin{pmatrix}
        \vone_r^T & \va_i^T & \va_i^T\va_i
      \end{pmatrix}^T,\\
      \mR_k' &= -\begin{pmatrix}
        \vzero_{r\times r} & -2\diag{\vr_k} & \ve_{1,r} \\
        2\diag{\vr_k} &-2\mI_{r\times r} & 0_{r\times1} \\
        \ve_{1,r}^T & \vzero_{1\times r} & 0
      \end{pmatrix},
    \end{split}
  \end{equation}
  we can now verify\footnote{A more
    detailed derivation can be found in the online appendix.} that
  \begin{align*}
    s_k^{m_R}(i,j) \le s_k^{m_R}(i',j')
      &\iff s_k^{m_T}(i,j) \le s_k^{m_T}(i',j'),
  \end{align*}
  which implies that $m_T$ and $m_R$ agree on the ranking for each relation,
  i.e., $\pi(\tS^{m_T})=\pi(\tS^{m_R})$. Since
  $m_R\in M^{\text{RESCAL}}_{2r+1}$, we obtain
  $\pi(\tS^{m_T})\in \pi(\cM^\text{RESCAL}_{2r+1})$ as claimed.
\end{proof}

The proof above shows that TransE can be viewed as a bilinear model with the
constraints specified in Eq.~\eqref{eq:R_from_T}.

\begin{theorem}\label{th:sub_T_R}
  $M^{\text{TransE}}$ does not subsume $M_{r}^{\text{RESCAL}}$ for any $r\ge2$.
\end{theorem}
Note that the theorem implies that there are RESCAL models with $r=2$ that
cannot be expressed with any TransE model, no matter how large its size.
\begin{proof}
  Fix some $r\ge 2$ and consider the RESCAL model $m_R\in M_{r}^{\text{RESCAL}}$
  specified by parameters
  \begin{align*}
    \va_i' &= \begin{cases}
      \ve_{1,r} & \text{for $i=1$} \\
      \ve_{2,r} & \text{for $i=2$} \\
      \vzero_{r} & \text{otherwise}
    \end{cases}, \\
    \mR_k' &= \begin{cases}
      \begin{pmatrix}
        1 & 1 & \vzero_{r-2} \\
        1 & 0 & \vzero_{r-2} \\
        \vzero_{(r-2)\times 1} & \vzero_{(r-2)\times 1} & \vzero_{(r-2)\times
          (r-2)}
      \end{pmatrix} & \text{for $k=1$} \\
      \vzero_{r\times r} &\text{otherwise}
    \end{cases}
  \end{align*}
  We have $s_1^{m_R}(1,1) = 1$, $s_1^{m_R}(2,2) = 0$. Thus
  $s_1^{m_R}(1,1)\neq s_2^{m_R}(2,2)$ and consequently
  $\pi_{11}(\mS^{m_R}_{(1)})\neq \pi_{22}(\mS^{m_R}_{(1)})$. Now pick any TransE
  model $m_T\in\cM^\text{TransE}$, denote by $\mA$ and $\mR$ its parameters, and
  observe that $s_k^{m_T}(1,1) = s_k^{m_T}(2,2)=-\norm{\vr_k}_2^2$. Thus
  $\pi_{11}(\mS^{m_T}_{(1)})= \pi_{22}(\mS^{m_T}_{(1)})$. Since this holds for
  any TransE model, we conclude that
  $\pi(\tS^{m_R})\not\in \pi(\cM^{\text{TransE}})$.
\end{proof}

\citet{NickelRP16} argued that HolE can be viewed as a compressed version of
RESCAL and implicitly established the subsumption relationship to RESCAL. We
present their argument formally below.
\begin{theorem}\label{th:sub_R_H}
  $M^{\text{RESCAL}}_r$ \text{subsumes} $M^\text{HolE}_{r}.$
\end{theorem}
\begin{proof}
  From the definition of HolE, we rewrite
  \begin{align*}
    \vr^T_k(\va_i \star \va_j)
             & = \sum_{t=1}^{d} r_{kt} \sum_{u=1}^{d}a_{iu} a_{j({(t+u-2 }\ \text{mod}\ r) +1)} \\
             & = \sum_{u=1}^{d} a_{iu} \sum_{t=1}^{d} r_{k((t-u\smod r)+1)}  a_{jt} \\
             & = \va_i^T \mR_k \va_j,
  \end{align*}
  where $\mR_k=
  \begin{pmatrix}
    r_{k1}   & r_{k2} & \dots   & r_{kr} \\
    r_{kr}   & r_{k1} & \dots  & r_{k(r-1)} \\
    \vdots   & \vdots & \ddots  & \vdots \\
    r_{k2}   & r_{k3} & \dots   & r_{k1} \\
  \end{pmatrix}
  $.
\end{proof}

Recently, \citet{HayashiS17} proved that
$\cM^{\text{HolE}}_{2r+1} \supseteq \cM^{\text{ComplEx}}_r$ and
$\cM^{\text{HolE}}_{r} \subseteq \cM^{\text{ComplEx}}_r$. Putting this together
with Th.~\ref{th:sub_R_H}, we obtain:
\begin{corollary}
  $M^{\text{RESCAL}}_{2r+1}$ subsumes $M^{\text{ComplEx}}_r$.
\end{corollary}

Finally, since DISTMULT differs from RESCAL only in that DISTMULT adds a
diagonality constraint, we directly obtain:
\begin{theorem}
  $M^{\text{RESCAL}}_{r}$ subsumes $M^{\text{DISTMULT}}_r$.
\end{theorem}

\subsection{Universality}

We say that class $M^t$ is \emph{universal} if
$\pi(\cM^t)=\pi(\bR^{N\times N\times K})$, i.e., any ranking tensor can be
expressed. As with subsumption, universality does by no means imply that a model
class is suitable for use in practice. If a model class is not universal,
however, care must be taken because certain relations cannot be modeled.

A direct consequence of Th.~\ref{th:sub_T_R} is:
\begin{corollary}
  $M^\text{TransE}$ is not universal.
\end{corollary}

We establish the universality of RESCAL, HolE, and ComplEx next.
\begin{theorem}\label{th:R_universal}
  $M^\text{RESCAL}_N$ is universal.
\end{theorem}
\begin{proof}
  Pick any ranking tensor $\tP\in\pi(\bR^{N\times N\times K})$. Consider the
  model $m\in M^\text{RESCAL}_N$ with parameterization $\mA=\mI_N$ and
  $\mR_k = -\tP_{(k)}$. Then $\mS^m_k = \mA\mR_k\mA^T = -\tP_{(k)}$ and thus
  $\tS^m=-\tP$. Using the fact that $\pi(-\tP)=\tP$, we conclude that
  $\tP\in \pi(\cM^\text{RESCAL}_N)$.
\end{proof}

Note that models in $M^\text{RESCAL}_N$ have very large embeddings. It is more
involved to establish whether or not $M^\text{RESCAL}_r$ is universal for some
$r<N$. We approach this question below and show that $r$ needs to be linear in
$N$ to obtain universality.

\begin{theorem}\label{th:R_not_universal}
  $M^\text{RESCAL}_{\floor{N/32-1}}$ is not universal.
\end{theorem}

The proof (given below) makes use of the notion of rounding
rank~\cite{0003GM16}. Given a \emph{rounding threshold} $\tau \in \mathbb{R}$,
denote by
\[
  \round_\tau(x) = \begin{cases}
    1 & \text{if} \ x \geq \tau\\
    0 & x < \tau
  \end{cases}
\]
the rounding function. We apply $\round_\tau$ to matrices and tensors by
rounding element-wise. In particular, when $\mA \in \mathbb{R}^{m\times n}$ is
any real-valued matrix, then $\round_\tau (\mA)$ is the $m \times n$ binary
matrix with $[\round_\tau (\mA)]_{ij} = \round_\tau (a_{ij} )$. We assume
$\tau=1/2$ unless explicitly stated otherwise and write $\round$ for
$\round_{1/2}$.

\begin{definition}
  For $\tau \in \mathbb{R}$, the \emph{rounding rank w.r.t.~$\tau$} of a binary
  matrix $\mB\in\sset{0,1}^{m\times n}$ is given by
  \[
    \text{rrank}_\tau (\mB) = \min \sset{ \text{rank}(\mA) : \mA \in
      \mathbb{R}^{m\times n} ,\round_\tau (\mA) = \mB }.
  \]
\end{definition}

Given a Boolean matrix $\mB$, say that a scoring matrix $\mS$ is
\emph{consistent} with $\mB$ if
\begin{equation}
  \label{eq:desire1}
  b_{ij} =1 \ \text{and} \ b_{i'j'} = 0 \implies  \pi_{ij}(\mS) > \pi_{i'j'}(\mS).
\end{equation}

The rounding rank can be interpreted as the minimum rank of a scoring matrix
that is consistent with $\mB$.
\citet{0003GM16} proved that the rounding rank differs by at most 1 for different
choices of $\tau$ and that it is connected to the \emph{sign
  rank}~\cite{AlonMY16}.  The rounding rank can be much smaller than the matrix
rank in practice, which partially explains the the success of bilinear models.

\begin{proof}[Proof (of Th.~\ref{th:R_not_universal})]
  \citet{AlonFR85} showed that there exist Boolean matrices in
  $\sset{0,1}^{N\times N}$ with rounding rank at least $N/32$ for every
  $N$. Pick any such matrix $\mB$. The proof is by contradiction. Consider $K=1$
  and suppose there exists a scoring matrix
  $\mS\in M^\text{RESCAL}_{\floor{N/32-1}}$ that satisfies
  Eq.~\eqref{eq:desire1}. Observe that $\mS$ has rank at most $\floor{N/32-1}$
  because it is defined by a product involving a
  $\floor{N/32-1}\times\floor{N/32-1}$ matrix. But this implies that
  $\rrank(\mB)\le N/32-1$, a contradiction.
\end{proof}
Note that the proof implies that there exists ranking tensors with just two
distinct ranks that cannot be expressed by
$M^\text{RESCAL}_{\floor{N/32-1}}$. Since RESCAL is an unconstrained bilinear
model, we can generalize to other model classes.

\begin{corollary}
  No model class that only contains bilinear models of size less than
  $\frac{N}{32}$ is universal.
\end{corollary}

\begin{theorem}\label{complexuni}
  $M^\text{ComplEx}_{KN}$ and $M^\text{HolE}_{2KN+1}$ are universal.
\end{theorem}
\begin{proof}
  Pick any scoring tensor $\tS\in\bR^{N\times N\times K}$. \citet{Trouillon2016a}
  showed that for every $N\times N$ real matrix and thus every scoring matrix
  $\mS_k$, there exists $\mA_k,\mD_k \in \mathbb{C}^{N\times N}$, where $\mA_k$
  is unitary, $\mD_k$ diagonal, and $\mS_k = \Real(\mA_k\mD_k\mA_k^*)$. Now
  consider the ComplEx model with
  \begin{align*}
    \mA&=
         \begin{pmatrix}
           \mA_1 & \mA_2 & \cdots & \mA_K
         \end{pmatrix} \\
    \mR_k &= \diag{\vzero_{N\times N}, \ldots, \mD_k, \ldots, \vzero_{N\times
            N}}
  \end{align*}
  We can verify $\mS_k=\text{Re}(\mA\mR_k\mA^*)$ for each $k$. Thus
  $\tS\in\cM^\text{ComplEx}_{KN}$ and it follows that $M^\text{ComplEx}_{KN}$ is
  universal. The universiality of $M^\text{HolE}_{2KN+1}$ follows from the fact
  that $\cM^{\text{HolE}}_{2r+1} \supseteq \cM^{\text{ComplEx}}_r$ for every
  $r$~\cite{HayashiS17}.
\end{proof}
Finally, since DISTMULT's relation matrix is diagonal and thus symmetric, DISTMULT cannot model asymmetric relations.
\begin{theorem}
  $\cM^\text{DISTMULT}$ is not universal.
\end{theorem}

\subsection{Consistent Ranking}

Suppose we are given an $N\times N\times K$ Boolean tensor $\tB$ and we look for a ranking
tensor $\tP$ that is consistent with $\tB$ in each frontal slice, i.e.,
$p_{ijk}<p_{i'j'k}$ whenever $b_{ijk}=1$ and $b_{i'j'k}=0$. In this section, we
establish upper bounds on the size\footnote{The expressive power of
  models considered here is non-decreasing as their size grows.}
that various bilinear models need to express a ranking that is consistent with
$\tB$, i.e., which ranks $1$s above $0$s. Here we think of $\tB$ as the correct
completed KB; there is no hope for a model class not consistent with $\tB$ to
recover the correct KB.

Note that even if a model class is not universal, it may still contain
consistent models for all Boolean tensors. This is not the case for DISTMULT and
TransE, however. In particular, since DISTMULT produces symmetric scoring
matrices, DISTMULT does not contain models consistent with any Boolean tensor
that has an asymmetric frontal slice. For TransE, the proof of
Th.~\ref{th:sub_T_R} implies that TransE does not contain models for Boolean
tensors with both 0s and 1s on the main diagonal of any of its frontal slices.

\begin{theorem}
  There exists Boolean tensors $\tB$ such that no ranking tensor in
  $\pi(\cM^\text{DISTMULT})$ is consistent with $\tB$.
\end{theorem}

\begin{theorem}
  There exists Boolean tensors $\tB$ such that no ranking tensor in
  $\pi(\cM^\text{TransE})$ is consistent with $\tB$.
\end{theorem}

For RESCAL, which is universal, we can make use of the rounding-rank decomposition to obtain a tighter bound than the one implied by its universality.

\begin{theorem} \label{th:rrank_R} For any boolean tensor $\tB$,
  $\pi(\cM_r^\text{RESCAL})$ contains a ranking tensor consistent with $\tB$ if
  \[ r\ge \min\sset{N,\; 2\sum_{k=1}^K \rrank(\mB_{k})}.\]
\end{theorem}
\begin{proof}
  The case $r\ge N$ follows from Th.~\ref{th:R_not_universal}. Denote by $r_k$
  the rounding rank of slice $\mB_k$ of $\tB$; we explicitly construct a
  consistent RESCAL model with $r=2\sum_k r_k$ (as asserted). To do so, pick any
  $\mL_k,\mQ_k\in \mathbb{R}^{N\times r_k}$ that form a \emph{rounding-rank
    decomposition} of $\mB_k$, i.e., for which $\mB_k=\round(\mL_k\mQ_k^T)$. (By
  the definition of rounding rank, such matrices always exist.) Now set
  \begin{align*}
    \va^T_i &= \begin{pmatrix}~
      [\mL_1]_{i:} & [\mQ_1]_{i:} & \cdots & [\mL_K]_{i:} & [\mQ_K]_{i:} 
    \end{pmatrix}^T  \\
    \mM_{k} &= \begin{pmatrix}
      \vzero_{r_k\times r_k}  & \mI_{r_k\times r_k}\\
      \vzero_{r_k\times r_k} & \vzero_{r_k\times r_k}
    \end{pmatrix} \\
    (\mR_k)_{ij}&= \diag{\vzero_{2r_1\times 2r_1}, \ldots, \mM_{k}, \ldots,
                  \vzero_{2r_K\times 2r_K}}
  \end{align*}
  We can now verify that $\round(\mA\mR_k\mA^T)=\mB_k$, which implies
  consistency.
\end{proof}

\begin{theorem} \label{rh:rrank_C} For any Boolean tensor $\tB$,
  $\pi(\cM_r^\text{ComplEx})$ contains a ranking tensor consistent with $\tB$ if
  \[r\ge \min\sset{KN, 2\sum_{k=1}^K \rrank(\mB_{k})}.\]
\end{theorem}
\begin{proof}
  The case $r\ge KN$ follows directly from Th.~\ref{complexuni}. To obtain
  $r\ge2\sum_{k=1}^K \rrank(\mB_{k})$, define $r_k$, $\mL_k$, and $\mQ_k$
  as in the proof of Th.~\ref{th:rrank_R}, and set $\mS_k=\mL_k\mQ_k^T$. Then
  there exist matrices $\mA_k \in \mathbb{C}^{N\times 2r_k}$ and
  $\mD_k \in \mathbb{C}^{2r_k \times 2r_k}$, with $\mD_k$ being diagonal, such
  that $\mS_k = \Real(\mA_k\mD_k\mA_k^*)$~\cite{Trouillon2016a}. Now define
  \begin{align*}
    \mA&=
         \begin{pmatrix}
           \mA_1 & \mA_2 & \cdots & \mA_K
         \end{pmatrix} \\
    \mR_k &= \diag{\vzero_{2r_1\times 2r_1}, \ldots, \mD_k, \ldots,
            \vzero_{2r_K\times 2r_K}}
  \end{align*}
  and observe that $\mS_k=\Real(\mA\mR_k\mA^*)$.
\end{proof}

As a corollary of the above theorem, we have:
\begin{corollary} \label{rh:rrank_H} For any Boolean tensor $\tB$,
  $\pi(\cM_r^\text{HolE})$ contains a ranking tensor consistent with $\tB$ if
  \[r\ge \min\sset{2KN+1, 4\sum_{k=1}^K \rrank(\mB_{k})+1}.\]
\end{corollary}

\section{Training and Relation-Level Ensemble}

We have seen that various prior models can be interpreted as a bilinear models
subject to certain constraints. In other words, they are diverse with respect to
their expressivity. So far, we did not touch on how to select a suitable model
for a given dataset and from a given model class. In this section, we briefly
discuss model training in a margin-based framework. We then propose a simple
relation-level ensemble that combines multiple individual models. The rationale
behind using an ensemble is that whether a model class can represent well or be
trained well on a relation depends on properties of that relation. The ensemble
thus aims to pick the best model (or a combination of models) for each relation.

\subsection{Margin-Based Training}

We assume throughout that we are given a set of positive triples
$\cT^+\subset\cE\times\cR\times\cE$, but no negative evidence. This is a common scenario in
practice. To deal with the absence of negative evidence, ranking-based
frameworks aim to produce a model that ranks triples in $\cT^+$ higher than
other triples. A common approach~\cite{BordesUGWY13,NickelRP16} is to define a
set of ``negative'' triples for each positive triple $(i,k,j)\in\cT^+$ by
perturbing subject or object:
\begin{align*}
  \cT^-_{(i,k,j)}
  &=\sset{(i',k,j\mid i'\in\cE,\ (i',k,j)\notin \cT^+)} \cup \\
  &\quad \sset{(i,k,j'\mid j'\in\cE,\ (i,k,j')\notin \cT^+)}.
\end{align*}
This approach corresponds to a local closed-world
assumption~\cite{0001GHHLMSSZ14}. We now briefly summarize a common margin-based
framework for training~\cite{BordesUGWY13}. There are a number of alternatives,
including logistic loss~\cite{RiedelYMM13} and negative
log-likelihood~\cite{TrouillonWRGB16}. Margin-based frameworks often lead to
faster training times in practice because they focus on ``informative'' pairs of
positive and negative triples, i.e., they ignore parts of the data that are
already more or less well-represented by the model. In particular, we minimize
\begin{equation*}
\sum_{\substack{(i^+,k,j^+)\in \cT^+, \\ (i^-,k,j^-)\in \cT^-_{(i^+,k,j^+)}}}
\frac{\left[f(i^-,k,j^-)+\gamma- f(i^+,k,j^+)\right]_+}{\absn{\cT^-_{(i^+,k,j^+)}}},
\end{equation*}
where $0\le\gamma\in\bR^+$ is a \emph{margin hyperparameter}, $[x]_+=\max(0,x)$,
and $f$ depends on the model being trained. For all models but HolE, we set
$f(i,k,j)=s^m_k(i,j)$. For HolE, we set $f(i,k,j) = \sigma(s^m_k(i,j))$, where
$\sigma$ denotes the logistic function, as suggested by the authors. In our
experimental study, we also consider an additional $L_2$ regularization term
over the model parameters. The models can be fit using stochastic gradient
descent (SGD) as in~\cite{BordesUGWY13,LinLSLZ15}. The computational cost per
SGD step of RESCAL is $O(r^2)$, of HolE $O(r\log n)$, and of all other models
$O(r)$.

\subsection{Relation-Level Ensemble}

The simplest way of combining multiple models is to construct an ensemble at the
model level~\cite{krompassLD4KD2015}. Our experimental study suggests that the
relative performance of different models is relation-dependent, however. A more
promising approach is therefore to combine models at the relation level. To the
best of our knowledge, this simple approach has not been explored previously.

Our ensemble is based on stacking. A meta learner is used to combine the ranking
matrices produced by the individual models such that some accuracy measure is
maximized. Here we use logistic regression. To do so, we construct for each
relation a dataset that contains all of its positive triples as well as an equal
amount of negative triples obtained by randomly perturbing each positive triple
following the same strategy as in training individual models. For logistic
regression, we use rescaled scores of the individual models as features and the
positive/negative class label as response variable. Rescaling accounts for the
variety in range of scores of different models; we rescale each feature linearly
into range $[0,1]$~\cite[Sec.~3.5.2]{HanKP2011}.


\section{Experiments}

We conducted an experimental study on two real-world datasets, which are
commonly used in prior work on KB completion. The primary goal of our study was
to provide independent evidence for the performance of various bilinear models
under the margin-based ranking framework. We also evaluated relation-level
ensembles of such models and compared the results to prior results reported in
the literature (for bilinear and other models).

\subsection{Experimental Setup}

All datasets, experimental results, and source code will be made publicly
available.

\paragraph{Data.} We used the WN18~\cite{BordesGWB14} and
FB15K~\cite{BordesUGWY13} datasets, which were extracted from
WordNet~\cite{Miller95} and Freebase~\cite{BollackerEPST08},
respectively. WordNet contains words and their relationships. Freebase contains
various facts across a large number of relations. The two datasets are presplit
into a training set, a validation set, and a test set. Table~\ref{tab:data}
summarizes the key statistics.

\begin{table}
  \centering
  \caption{Dataset statistics}
  \label{tab:data}
  \begin{tabular}{lrrrrr}
    \hline
    Dataset & \# Ent.          & \# Rel.         & \# Train.         & \# Valid.        & \# Test          \\ \hline
    WN18    & \numprint{40943} & 18              & \numprint{141442} & \numprint{5000}  & \numprint{5000}  \\
    FB15K   & \numprint{14951} & \numprint{1345} & \numprint{483142} & \numprint{50000} & \numprint{59071} \\ \hline
  \end{tabular}
\end{table}

\paragraph{Methods and training.} We considered RESCAL (R), HolE (H), and TransE
(T) in our experimental study. We reimplemented each method in C++, partly using
the Intel Math Kernel Library. We trained each model in the margin-based ranking
framework using Adagrad~\cite{DuchiHS11}. In each step, we sampled a positive
triple at random and obtained a negative triple by randomly perturbing subject
or object. Sampling was done without replacement and we did not use
mini-batches.  When using the same hyperparameters, our implementation provided
similar or better fits than the original implementations provided by the
authors. Note that our study was limited in that we considered only one
particular training method; no conclusions can be drawn about other training
methods. We focused on margin-based ranking because it led to much faster
training times, making this study more feasible. We used LIBLINEAR
for logistic regression. 

\paragraph{Evaluation.} We evaluated model performance for the tasks of entity
ranking and triple classification on the test data. In \emph{entity ranking}, we
rank entities for queries of the form $R(?,e)$ or $R(e,?)$. Our evaluation closely follows
\citet{BordesUGWY13}, and we report \emph{mean reciprocal rank (MRR)},
\emph{HITS@10}, and \emph{mean rank (MR)} in the \emph{filtered} setting, i.e.,
predictions that correspond to tuples in the training or validation datasets
were discarded. In \emph{triple classification}, we are given a triple $(i,k,j)$
and are asked to classify it as positive or negative; we proceed
as~\citet{SocherCMN13} to produce the set of tuples to classify. To perform
classification, we determined a score threshold $\sigma_k$ for each relation and
model; scores larger than $\sigma_k$ were classified positive, else negative. We used
optimal thresholds with respect to the validation set.

\paragraph{Model selection.} Each of the models has a number of
hyperparameters. For all models, we trained the models solely on the training
data and used the validation data solely to tune hyperparameters. Test data was
not touched for model selection. We considered the following hyperparameter
settings: $r\in\{100, 200\}$, learning rate $\eta\in\{0.01, 0.1, 1\}$, weight of
L2-regularization $\lambda_e,\lambda_r\in\{0, 0.1, 0.01\}$ for entity and relation parameters,
resp., margin hyperparameter~$\gamma \in \{ 1, 2, 4, 8\}$ for RESCAL,
$\gamma \in \{ 0.2, 0.5, 0.7\}$ for HolE, and
$\gamma \in \{ 0.2, 0.5, 0.7, 1.0, 1.5\}$ for TransE.\footnote{We used smaller margins
  than the ones suggested for TransE with $L_1$
  distance~\cite{LinLSLZ15,WangZFC14,LinLLSRL15}. By doing this, we obtained
  comparable prediction performance as TransE-$L_1$.}

We performed exhaustive grid search, using 50 (2000 for TransE) epochs (passes
over the dataset) per hyperparameter setting and model. We then retrained the
best-performing setting (w.r.t.~HITS@10 on validation data) for each model on
the training data for up to 2,000 epochs. Tab.~\ref{tab:hyperparameters} reports
the hyperparameters ultimately selected.

\begin{table}
  \centering
  \captionof{table}{Hyperparameters settings used in our
    study}\label{tab:hyperparameters}
  \begin{tabular}{llccccc}
    \hline
    Dataset & Model  & $r$ & $\gamma$ & $\eta$ & $\lambda_{e}$ & $\lambda_{r
                                                             }$       \\\hline
WN18        & RESCAL & 200 & 1.0      & 0.10   & 0.10          & 0.01 \\
            & HolE   & 200 & 0.2      & 0.10   & 0.01          & 0.00 \\
            & TransE & 200 & 0.5      & 0.01   & -             & -    \\
     \hline
FK15k       & RESCAL & 200 & 4.0      & 0.10   & 0.10          & 0.01 \\
            & HolE   & 200 & 0.2      & 0.10   & 0.01          & 0.01 \\
            & TransE & 200 & 0.2      & 0.01   & -             & -    \\
    \hline
  \end{tabular}
\end{table}

\subsection{Results}

\begin{table*}
  \centering
  \caption{Entity ranking results of our experimental study. Best-performing entries marked bold.}
  \label{tab:en-results}
  \begin{adjustbox}{max width=\textwidth}
    \begin{tabular}{p{7cm}cccccc}
      \hline     
      Dataset                  & \multicolumn{3}{c}{WN18} & \multicolumn{3}{c}{FB15K}                                                           \\ \hline
    Model                      & HITS@10 (\%)             & MRR (\%)      & MR           & HITS@10 (\%)  & MRR (\%)      & MR                   \\ \hline
    HolE~\cite{NickelRP16}     & 94.1                     & 93.8          & 819          & 72.6          & 50.2          & 331                  \\
    TransE~\cite{BordesUGWY13} & 94.5                     & 43.9          & \textbf{474} & 79.5          & 34.4          & \phantom076          \\
      RESCAL~\cite{NickelTK11} & 87.8                     & 79.9          & 905          & 59.6           & 38.1          & 247                  \\ \hline
      RESCAL + TransE          & 94.8                     & 87.3          & 510          & 79.7          & 51.1          & \phantom061          \\
      RESCAL + HolE            & 94.4                     & \textbf{94.0} & 743          & 79.1          & 57.5          & 165                  \\
      HolE + TransE            & 94.9                     & 93.8          & 507          & 84.6          & 61.0          & \phantom067          \\
      RESCAL + HolE + TransE   & \textbf{95.0}            & \textbf{94.0} & 507          & \textbf{85.1} & \textbf{62.8} & \textbf{\phantom052} \\ \hline
    \end{tabular}
  \end{adjustbox}
\end{table*}

\begin{table*}
  \centering
  \caption{Entity ranking results as reported in the literature (not reproduced
    here, partly with different training methods, partly non-bilinear
    models). Entries marked ``-'' were not reported. Entries better than any
    result in our study are marked bold.}
  \label{tab:state-art}
  \begin{adjustbox}{max width=\textwidth}
    \begin{tabular}{p{7cm}ccccccc}
      \hline
      Dataset                                               & \multicolumn{3}{c}{WN18} & \multicolumn{3}{c}{FB15K}                                         \\ \hline
      Model                                                 & HITS@10 (\%)             & MRR (\%)      & MR           & HITS@10 (\%)  & MRR (\%)      & MR \\ \hline
      Gaifman~\cite{Niepert16}                              & 93.9                     & -             & \textbf{352} & 84.2          & -             & 75 \\ 
      ComplEx~\cite{TrouillonWRGB16}, r=150/200             & 94.7                     & \textbf{94.1} & -            & 84.0          & \textbf{69.2} & -  \\
      DISTMULT~\cite{TrouillonWRGB16}, r=150/200            & 93.6                     & 82.2          & 902          & 82.4          & \textbf{65.4} & 97 \\
      \mbox{R-GCN+DISTMULT~\cite{SchlichtkrullKB17}, r=200} & \textbf{96.4}            & 81.9          & -            & 84.2          & \textbf{69.6} & -  \\
      ANALOGY~\cite{LiuWY17}, r=200                         & 94.7		       & \textbf{94.2} & -            & \textbf{85.4} & \textbf{72.5} & -  \\\hline
    \end{tabular}
  \end{adjustbox}
\end{table*}

\begin{table}
  \centering
  \caption{Detailed entity ranking results (FB15k, HITS@10)}
  \label{tab:hits-ctg}
  \begin{adjustbox}{max width=\linewidth}
    \begin{tabular}{l@{\hspace{.5em}}l@{\hspace{.7em}}l@{\hspace{.7em}}@{\hspace{.7em}}l@{\hspace{.7em}}l@{\hspace{1em}}l@{\hspace{.7em}}l@{\hspace{.7em}}l@{\hspace{.7em}}l}
    \hline
    Task      & \multicolumn{4}{c}{Predict subject} & \multicolumn{4}{c}{Predict object}                                                                            \\ 
    Relations & 1:1                                 & 1:N           & N:1           & N:N           & 1:1           & 1:N           & N:1           & N:N           \\ \hline
    TransE    & 75.8                                & 91.9          & 41.4          & 82.2          & 75.5          & 51.1          & 91.9          & 84.7          \\
    HolE      & 80.4                                & 69.5          & 44.7          & 77.4          & 79.0          & 57.8          & 59.1          & 79.0          \\
    RESCAL    & 43.1                                & 75.7          & 17.7          & 62.0          & 42.4          & 21.3          & 79.2          & 65.8          \\ \hline
    R+H+T     & \textbf{87.5}                       & \textbf{94.3} & \textbf{55.2} & \textbf{86.7} & \textbf{87.0} & \textbf{65.0} & \textbf{93.3} & \textbf{89.4} \\ \hline
  \end{tabular}
  \end{adjustbox}
\end{table}

\begin{table}
  \centering
  \caption{Triple classification results (FB15K)}
  \label{tab:triple_classification}
  \begin{adjustbox}{max width=\linewidth}
    \begin{tabular}{lc@{\hspace{.7em}}c@{\hspace{.7em}}c@{\hspace{.7em}}c@{\hspace{.7em}}c@{\hspace{.7em}}c@{\hspace{.7em}}c}
      \hline
      Model    & T    & H    & R    & R+T  & R+H  & H+T  & R+H+T         \\ \hline
      Accuracy & 96.2 & 93.7 & 94.6 & 96.7 & 95.8 & 96.5 & \textbf{96.9} \\ \hline
    \end{tabular}
  \end{adjustbox}
\end{table}

\paragraph{Entity ranking.}
Our results are summarized in Tab.~\ref{tab:en-results}. Detailed results can be
found in Tab.~\ref{tab:hits-ctg}, where we measured HITS@10 per relation
category and per argument to be predicted as in~\cite{BordesUGWY13}.

For the individual models, our results indicate that model performance depends
on the relation category. No single model always performed best across all
categories. HolE and TransE generally performed better than RESCAL; here
constraints help. The relation-level ensembles generally improved performance
w.r.t.~HITS@10 and MRR. Performance of MR was not improved, however, mainly
because this metric is sensitive to low-ranked triples (which existed in HolE
and RESCAL predictions). Note that adding RESCAL to the ensemble was
helpful. Finally, the ensemble of RESCAL, TransE, and HolE performed best
w.r.t.~HITS@10 on all relation categories and for both datasets.

In Tab.~\ref{tab:state-art}, we compare to some recent results reported in the
literature. Note that training methods were different than the one used in our
study for some of these models, and that some models are not
bilinear. Nevertheless, a direct comparison indicates that a relation-level
ensemble of multiple bilinear models is competitive to the state-of-the-art.

\paragraph{Triple classification.} Tab.~\ref{tab:triple_classification}
summarizes the HITS@10 performance of each individual model and various
relation-level ensembles for triple classification on FB15k. The results are
generally in line with the results for entity ranking. A notable exception is
that RESCAL outperforms HolE here; we conjecture that this is due to HolE's high
MR on this dataset.


\section{Related Work}

We focus on recent embedding models that solely use the KB as input.  There are
a number of methods that modify TransE in one way or another:
TransH~\cite{WangZFC14} and TransR~\cite{LinLSLZ15} improve support symmetric
and many-to-one relations, TransG~\cite{0005HHZ15} adds refines relation
embeddings by semantic components, and PTransE~\cite{LinLLSRL15} adds
multiple-step relation paths. Gaifman~\cite{Niepert16} exploits structural
features in the form of Horn clauses to construct
embeddings. \citet{SocherCMN13} combined neural networks with
tensors. \citet{SchlichtkrullKB17} models relational data with graph
convolutional networks. ANALOGY~\cite{LiuWY17} is a recent bilinear model that
constrains relation embeddings be real normal matrices.
Finally, 
\citet{NickelJT14}
provided a rank bound for exact recovery of a Boolean tensor with RESCAL. Our
results differ in that we consider consistency, not exact recovery.


\section{Conclusion}

We studied the expressive power of and subsumption relationships between recent
bilinear embedding models for knowledge graphs. We introduced the concepts of
universality and consistency, which capture different aspects of model
expressiveness, and provided bounds on model sizes needed for universality or
consistency with a given dataset. We argued that using a relation-level
ensembles are beneficial for multi-relational learning. Finally, we conducted an
independent experimental study that compared various bilinear models in a common
setup.

Future work includes tightening the bounds provided here, studying which
relation types can be represented by which models, and exploring the
relationship between additional models. We also expect an in-depth study of
model performance with various alternative training methods to be insightful.


\clearpage
\makeatletter 
\let\@biblabel\@gobble
\makeatother
\bibliography{references}
\bibliographystyle{aaai}

\end{document}